\documentclass[10pt,twocolumn,letterpaper]{article}

\usepackage[pagenumbers]{cvpr} 

\usepackage{graphicx} 
\graphicspath{{figures/}}
\usepackage{amsmath}
\usepackage{amssymb}
\usepackage{amsthm}
\usepackage{thmtools}
\usepackage{booktabs}
\usepackage{enumitem}
\usepackage{multicol}
\usepackage{multirow}
\usepackage{rotating}
\usepackage[pagebackref,breaklinks,colorlinks]{hyperref}
\usepackage{xcolor}
\usepackage{symbols}

\usepackage[capitalize]{cleveref}
\crefname{section}{Sec.}{Secs.}
\Crefname{section}{Section}{Sections}
\Crefname{table}{Table}{Tables}
\crefname{table}{Tab.}{Tabs.}

\def\cvprPaperID{8284} 
\def\confName{CVPR}
\def\confYear{2023}

\begin{document}
\definecolor{darkred}{rgb}{0.7,0.1,0.1}
\definecolor{darkgreen}{rgb}{0.1,0.7,0.1}
\definecolor{cyan}{rgb}{0.7,0.0,0.7}
\definecolor{dblue}{rgb}{0.2,0.2,0.8}
\definecolor{maroon}{rgb}{0.76,.13,.28}
\definecolor{burntorange}{rgb}{0.81,.33,0}
\definecolor{tealblue}{rgb}{0.212,0.459, 0.533}

\definecolor{pp}{rgb}{0.43921569, 0.18823529, 0.62745098}
\definecolor{rr}{rgb}{0.5254902 , 0.00784314, 0.12941176}
\definecolor{bb}{rgb}{0.09019608, 0.23529412, 0.37647059}
\definecolor{yy}{rgb}{0.49803922, 0.3372549 , 0.0}
\definecolor{gg}{rgb}{0.02352941, 0.3372549 , 0.17647059}

\theoremstyle{definition}
\newtheorem{definition}{Definition}[section]
\newtheorem{conjecture}{Conjecture}[section]
\newtheorem{example}{Example}[section]
\newtheorem{theorem}{Theorem}[section]
\newtheorem{lemma}{Lemma}[section]
\newtheorem{assumption}{Assumption}[section]
\newtheorem{proposition}{Proposition}[section]
\newtheorem{corollary}{Corollary}[section]
\newtheorem{claim}{Claim}[section]
\newtheorem{fact}{Fact}[section]
\renewcommand{\qedsymbol}{$\blacksquare$}

\ifdefined\ShowNotes
  \newcommand{\colornote}[3]{{\color{#1}\bf{#2: #3}\normalfont}}
\else
  \newcommand{\colornote}[3]{}
\fi

\newcommand{\eat}[1]{} %
\newcommand{\myparagraph}[1]{{\vspace{0.15cm}\noindent\bf #1}}

\newcommand{\lce}{\ell_{\text{CE}}}

\newcommand{\product}[1]{\langle #1 \rangle}
\newcommand{\rot}[1]{\rotatebox{90}{#1}} 

\title{Understanding and Constructing Latent Modality Structures \\ in Multi-Modal Representation Learning}
\author{Qian Jiang$^1$, ~Changyou Chen$^{2,3}$, ~Han Zhao$^{1,3}$, ~Liqun Chen\thanks{Work done while at Amazon.}, ~Qing Ping$^3$, \\ ~Son Dinh Tran$^3$,  ~Yi Xu$^3$, ~Belinda Zeng$^3$, ~Trishul Chilimbi$^3$\\
$^1$University of Illinois at Urbana-Champaign \qquad $^2$University at Buffalo \qquad
$^3$Amazon \\
{\tt\small qianj3@illinois.edu \qquad lqchen06@outlook.com}\\
{\tt\small \{vchencha, uhanzhao, pingqing, sontran, yxaamzn, zengb, trishulc\}@amazon.com}
}
\maketitle

\begin{abstract}
Contrastive loss has been increasingly used in learning representations from multiple modalities. In the limit, the nature of the contrastive loss encourages modalities to exactly match each other in the latent space. Yet it remains an open question how the modality alignment affects the downstream task performance. In this paper, based on an information-theoretic argument, we first prove that exact modality alignment is sub-optimal in general for downstream prediction tasks. Hence we advocate that the key of better performance lies in meaningful latent modality structures instead of perfect modality alignment. To this end, we propose three general approaches to construct latent modality structures. Specifically, we design 1) a deep feature separation loss for intra-modality regularization; 2) a Brownian-bridge loss for inter-modality regularization; and 3) a geometric consistency loss for both intra- and inter-modality regularization. 
Extensive experiments are conducted on two popular multi-modal representation learning frameworks: the CLIP-based two-tower model and the ALBEF-based fusion model. We test our model on a variety of tasks including zero/few-shot image classification, image-text retrieval, visual question answering, visual reasoning, and visual entailment. Our method achieves consistent improvements over existing methods, demonstrating the effectiveness and generalizability of our proposed approach on latent modality structure regularization.

\end{abstract}
\section{Introduction}
\label{sec:intro}

Vision-language representation learning aims to learn generic representations from images and texts that could benefit multimodal downstream applications. As the two modalities are essentially from different data sources and distributions, how to effectively fuse the two modalities has become an important question. Some work aims to unify the representations of two modalities in one encoder, where the image and text are usually tokenized into sequences ~\cite{Wang2021UFOAU,Wang2022UnifyingAT,Wang2022ImageAA, Wang2021VLMoUV}. Another line of research represents the image and text modality separately with modality-specific encoders and utilizes contrastive learning to align the modalities, achieving state-of-the-art performance on multiple downstream applications \cite{radford2021clip,mokady2021clipcap,shen2021much,jia2021align,li2021albef, jiali2021codis, yang2022vision, Shukor2022EfficientVP,Kwon2022MaskedVA}.

\begin{figure}[t]
    \centering
    \includegraphics[width=\linewidth]{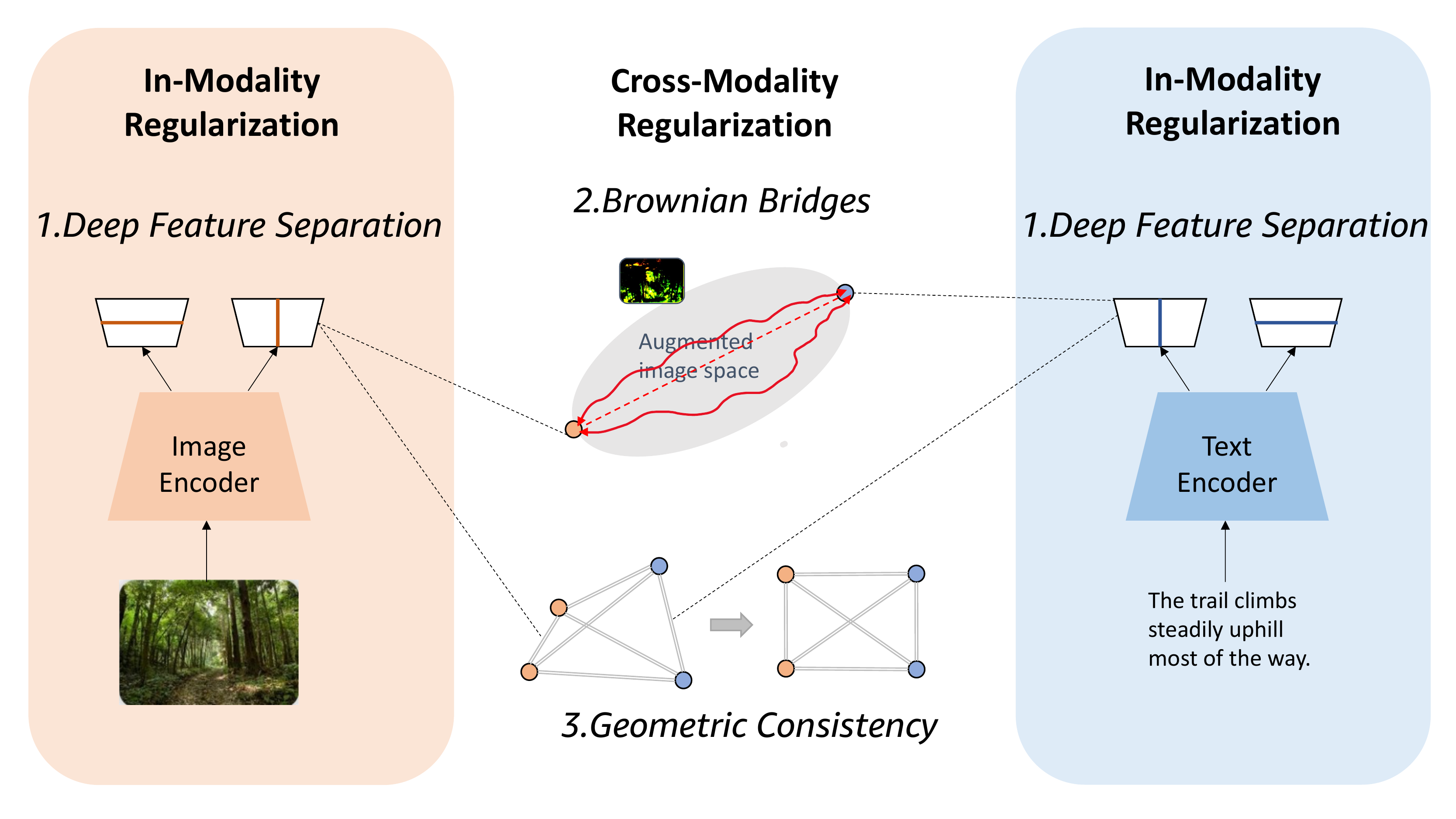}
    \vspace{-0.3cm}
    \caption{Constructing latent modality structures to improve multi-modal representation learning.}
    \label{fig:modality_reg}
    \vspace{-0.5cm}
\end{figure}

Despite the successful empirical practice of contrastive loss in multi-modal representation learning, it remains an open question whether bridging and aligning the two modalities always brings benefits to downstream tasks. One concept closely related to this question is the modality gap~\cite{zhang2020contrastive,radford2021clip,xu2021videoclip,liang2022mind}, where it is defined as the distance between the feature distributions of the two modalities. Modality alignment can be considered as reducing the modality gap. At a first glance, one would conjecture that contrastive loss would reduce the modality gap by pulling positive (paired) image and text data together for better representation. However, a recent study~\cite{liang2022mind} shows evidence that contrastive learning does not always reduce the modality gap. Furthermore, we also show in our empirical analysis that a reduced modality gap does not always guarantee better performance in downstream applications. Motivated by these empirical observations, in this paper we first theoretically study the modality gap problem, by showing that when the modality gap is zero, i.e., exact alignment between the two modalities, the learned representations necessarily have to pay a price for the downstream prediction task, which we term as the \emph{information gap} between the two modalities (Theorem~\ref{thm:gap}). Intuitively, this is because that representations with zero modality gap can only preserve predictive information present in \emph{both} of the modalities at the cost of losing the \emph{modality-specific} information. 

Our theory then suggests that instead of exact modality matching, whether learned representations are meaningful is an important factor in multi-modal representation learning. In particular, we propose to improve on top of contrastive learning with regularizations to construct better latent structures. We consider intra-modality, inter-modality, and intra-inter-modality regularizations. These regularizations are generalizable and can be applied to various vision-language models with modality-specific encoders. Specifically, for intra-modality regularization, motivated by our theoretic result, we propose deep feature separation to encourage the model to preserve both the modality-shared and modality-specific information in different components. For inter-modality regularization, we aim to bridge two modalities with their augmentations. Consequently, we proposed a Brownian bridge loss between the triplet of (text, augmented image, image) to regularize the inter-modality structures. For intra-inter-modality regularization, we introduce the geometric consistency loss that promotes geometric symmetry in the latent space. 
In summary, the main contributions of this paper are:
\begin{itemize}
  \item We conduct empirical and theoretical analysis on understanding the impact of the modality alignment on downstream tasks. We show that a reduced modality gap does not always guarantee better performance, and can instead hurt the performance when the information gap between the two modalities is large (Theorem~\ref{thm:gap}). Combined with the existing theory of contrastive learning, our theory suggests preserving both modality-shared and modality-specific information.
  \item Inspired by our theory, we propose three instrumental regularizations on top of the contrastive loss, {\it i.e.}, the intra-modality, inter-modality, and intra-inter-modality regularizations to improve latent modality structures.
  \item We conduct extensive and comprehensive experiments on various vision-language models to show that the proposed methods consistently improve over the baselines for different model families ({\it e.g.}, CLIP and ALBEF) and for different downstream applications ({\it e.g.}, cross-modality retrieval, VQA, VR and {\it etc}).
\end{itemize}



\section{Related work}
\label{sec:rel}


Most recent works on vision-language representation learning can be categorized based on how information from different modalities is used for joint learning. 
The first category applies unified models~\cite{Wang2021UFOAU,Wang2022UnifyingAT,Wang2022ImageAA, Wang2021VLMoUV} to process both images and texts, where the inputs are usually tokenized into sequences~\cite{Peng2022BEiTVM,Bao2022BEiTBP}. Unified models feature simpler and more universal designs, but typically underperform methods with modality-specific encoders (the second category). 
These methods use separate encoders for images and texts~(\eg CLIP\cite{radford2021clip,mokady2021clipcap,shen2021much}, ALIGN\cite{jia2021align}), and rely on contrastive loss~\cite{oord2018representation,he2020momentum,chen2020simple} to align multiple modalities. These methods have been shown to achieve state-of-the-art (SOTA) performance on image-text retrieval; but  the support is lacking for multi-modality tasks requiring inter-modality interaction, \eg VQA. 
To conquer this problem, most recent approaches use a hybrid fashion where the models have separate encoders for images and texts along with a late-fusion multi-modal encoder \cite{li2021albef, jiali2021codis, yang2022vision, Shukor2022EfficientVP,Kwon2022MaskedVA}. Specifically, image-text matching~(ITM) loss and masked language modeling~(MLM) loss are usually applied for training the fusion encoder.

The methods in the later category utilize separate encoders for different modalities. However, this can lead to the phenomenon that image embeddings and text embeddings reside in different regions of the joint latent space. Such a phenomenon, termed {\em modality gap}, is observed in many multi-modal models~\cite{zhang2020contrastive,radford2021clip,xu2021videoclip}. A recent study~\cite{liang2022mind} shows that the modality gap presents from the initialization and can be preserved during contrastive training. This naturally brings in another variety in multi-modality models -- the latent modality gap and modality structures. CyCLIP~\cite{Goel2022CyCLIPCC} advocates for the benefit of consistency in latent modality structures. Yet to the best of our knowledge, no other prior work has studied the modality gap from a theoretical view. In this work, we show that directly reducing the modality gap does not help in performance gain from both empirical experiments and theoretical analysis. Consequently, we propose to study the impact of latent modality structures, and propose three approaches to obtain more meaningful latent modality structures that can improve downstream applications.

\section{Understanding the Impact of Modality Gap on Downstream Performance}
\label{sec:prelim}
Despite being used extensively as a heuristic in practice~\cite{yang2022vision,zhang2020contrastive,liang2022mind,xu2021videoclip}, it remains an open question whether modality alignment in the feature space through contrastive learning is optimal for downstream performance~\cite{liang2022mind}. In this section, we first formally formulate the modality gap problem, present our empirical evidence on the relationship between the modality gap and the performance of downstream tasks, and then probe into its theoretical underpinning by providing an information-theoretical analysis.

\paragraph{Notation}
Throughout the paper, we will use $X_T$ and $X_V$ to denote the random variables corresponding to the input texts and images, respectively. We shall use $Y$ to denote the target variable in the downstream task of interest. For example, in the context of online shopping, $X_T$ and $X_V$ could be the textual and visual descriptions of a product, and in this case $Y$ is the expected sale of this product. When dealing with data with multi-modalities, we often use modality-specific encoder $g_T$ and $g_V$ to obtain features in the same latent space,
i.e., $Z_T = g_T(X_T)$ and $Z_V = g_V(X_V)$ are the extracted features from textual and visual inputs. In this work, we focus on the setting where inputs from different modalities are paired with each other, meaning that a sample consists of the tuple $(x_T, x_V, y)$ from the underlying joint distribution $p$. The goal of reducing the modality gap in the latent space is then to shrink the statistical distance (e.g., KL-divergence, etc) between $Z_T$ and $Z_V$.

For two random variables $X_T$ and $X_V$, we define $I(X_T;X_V)$ to be the Shannon mutual information between $X_T$ and $X_V$. Similarly, we use $H(Y\mid X_T, X_V)$ to denote the conditional entropy of $Y$ given the two modalities as input. Following common practice, for classification tasks, $\lce(\hat{y}, y)$ is the cross-entropy loss between the prediction $\hat{y}$ and the ground-truth label $y$. One useful fact about the conditional entropy $H(Y\mid X_T, X_V)$ and the cross-entropy loss is the following variational form~\cite{farnia2016minimax,zhao2020fundamental}: $H(Y\mid X_T, X_V) = \inf_{f}\sE_p[\lce(f(X_T, X_V), Y)]$, where the infimum is over all the prediction functions that take both $X_T$ and $X_V$ as input to predict the target $Y$ and the expectation is taken over the joint distribution $p$ of $(X_T, X_V, Y)$.

\subsection{Empirical Analysis on Modality Gap}
Given paired multi-modal data, one natural idea explored in the literature~\cite{yang2022vision,zhang2020contrastive,liang2022mind} is to use contrastive pretraining by treating paired multimodal data as the positive pairs and others as negative pairs. The goal is to align the positive pairs so that they are closer to each other in the feature space while at the same time ensuring the negative pairs to be farther away. More specifically, let $(x_T, x_V, y)$ and $(x'_T, x'_V, y')$ be two tuples sampled from the joint distribution. Then, in order to align the two modalities, $(x_T, x_V)$, $(x'_T, x'_V)$ are used as positive pairs while $(x_T, x'_V)$ and $(x'_T, x_V)$ are constructed as negative pairs.

Based on the contrastive loss principle~\cite[Theorem 1]{wang2020understanding}, a better model should come with smaller modality gaps (better alignment). However, despite being extensively used as a pretraining strategy in practice, it is unclear how the modality alignment affects the downstream tasks of interest. To approach this important question, we first conduct experiments to explore the effect of reducing modality gap on the task of image/text retrieval.

\begin{figure}[t]
  \centering
  \begin{subfigure}{0.32\linewidth}
    \centering
    \includegraphics[width=1.1in]{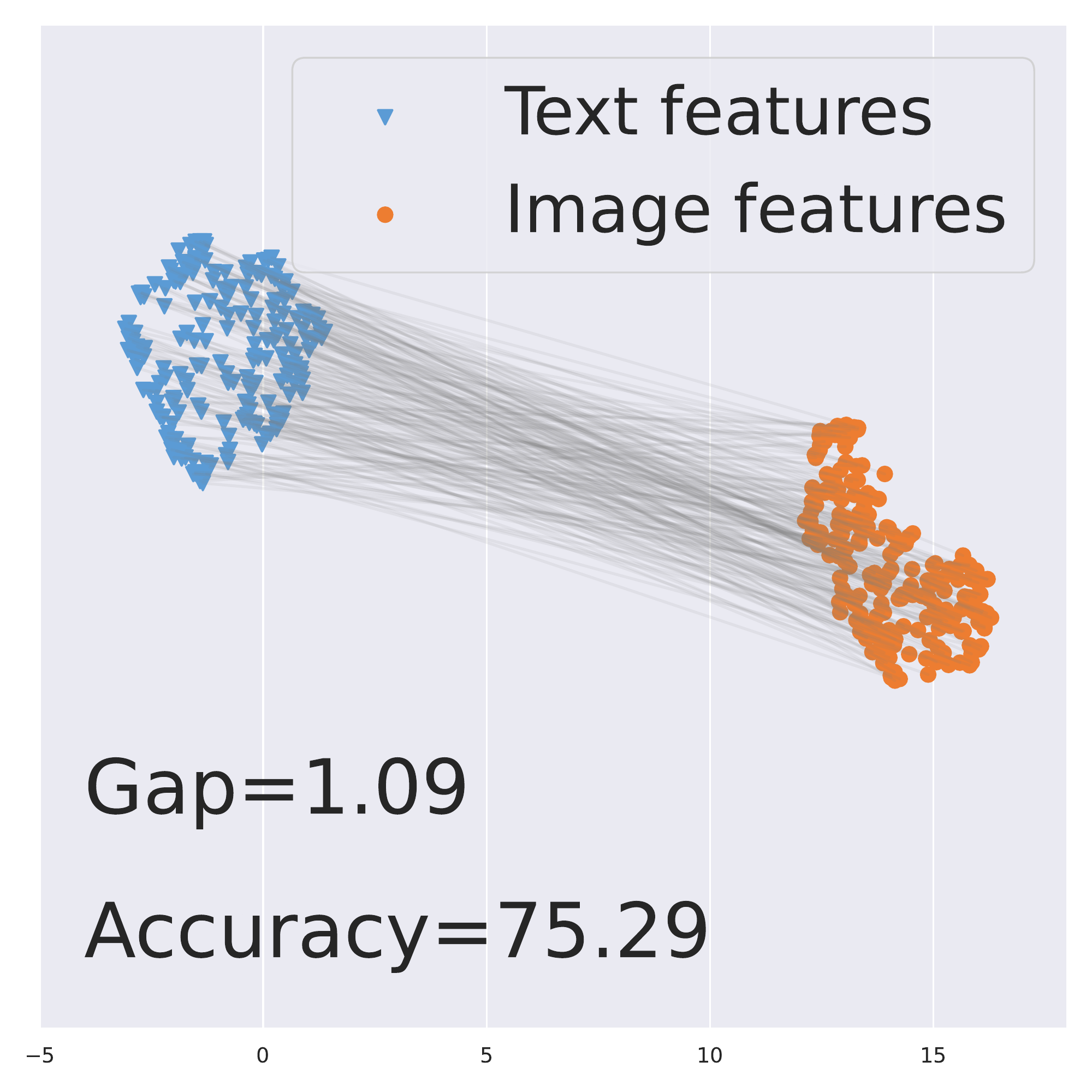}
  \end{subfigure}
  \begin{subfigure}{0.32\linewidth}
    \centering
    \includegraphics[width=1.1in]{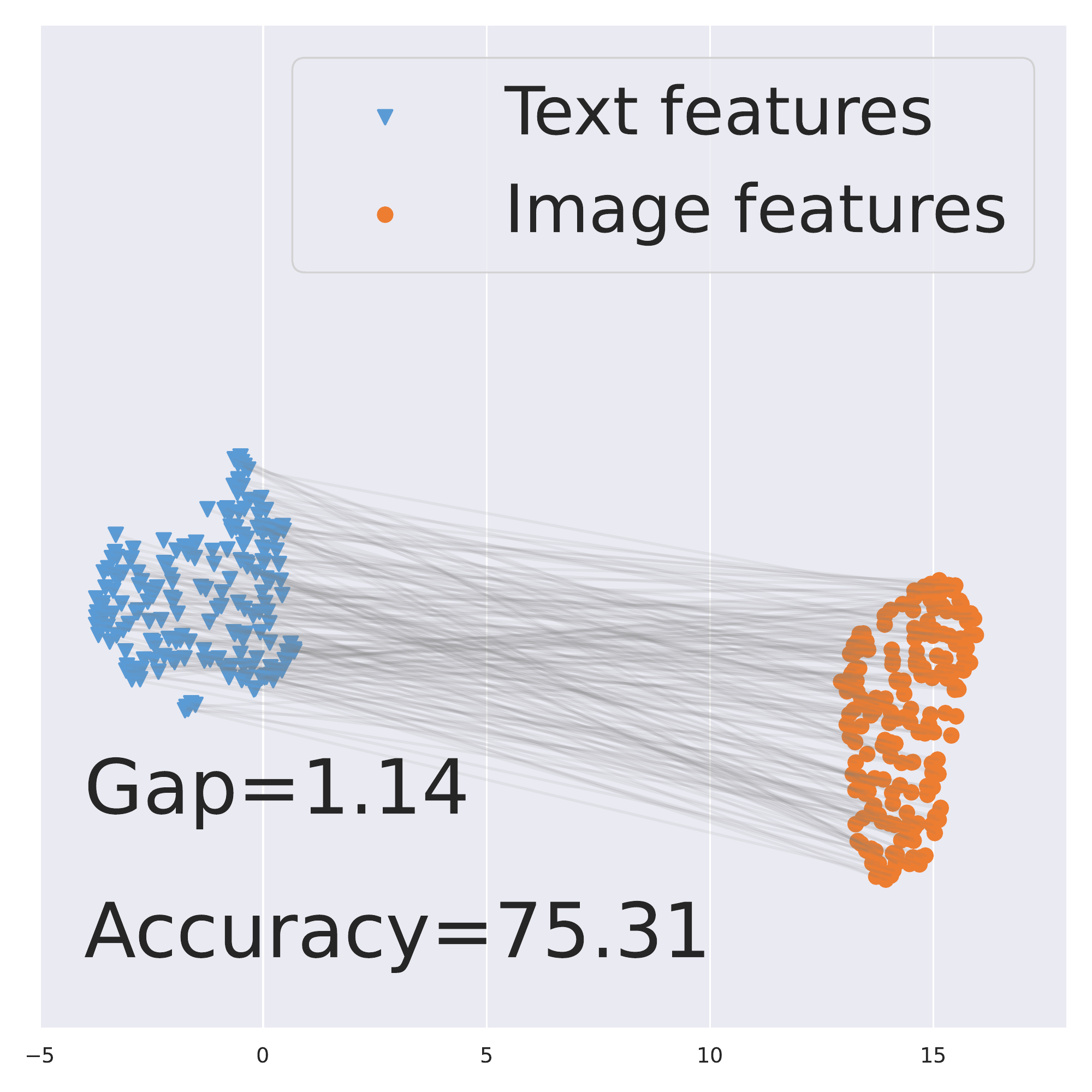}
  \end{subfigure}
   \begin{subfigure}{0.32\linewidth}
    \centering
    \includegraphics[width=1.1in]{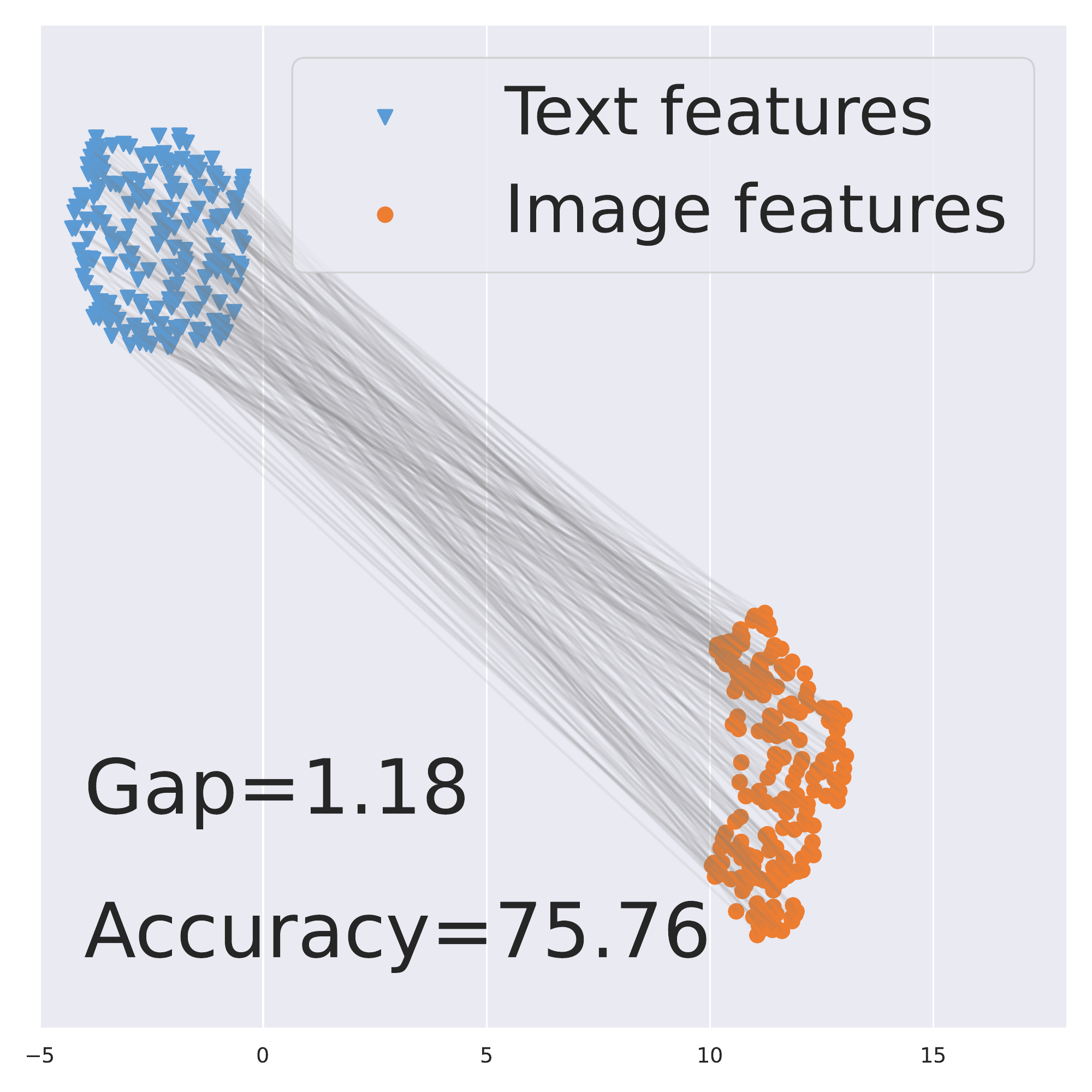}
  \end{subfigure}
  \vspace{-0.3cm}
  \caption{Visualization of the modality gap between text and image features. There is no clear-cut relationship between the gap of these two modalities and the downstream retrieval performance.}
  \label{fig:gap}
  \vspace{-0.5cm}
\end{figure}
We plot the alignment between paired image/text data in the feature space and also compute the average distance between them as the \emph{gap} measure in Fig.~\ref{fig:gap}. We perform pre-training on COCO~\cite{lin2014coco} dataset and evaluate the zero-shot retrieval performance on Flick30K~\cite{Young2014flickr} test set. We optimize an additional alignment loss during training, $\mathcal{L}_{\text{Align}} = 1/\langle Z_T, Z_V\rangle^2$, to reduce the gap between modalities.
We control the gap by adjusting the scale of the alignment loss. From Fig.~\ref{fig:gap}, we can see that the retrieval performance barely changes when changing the gap between two modalities. Note that as we normalized the data in the feature space, the gap difference in the figure is significant.

\subsection{An Information-Theoretic Analysis on Modality Gap}
Inspired by the empirical observation, we conjecture that \emph{reducing the modality gap in feature space does not always lead to better downstream task performance}. Nevertheless, it is instructive to theoretically understand when and in what kind of downstream tasks reducing the modality gap could help. To do so, we first define the \emph{information gap} $\Delta_p:= |I(X_T;Y) - I(X_V;Y)|$ to characterize the gap of utility provided by two modalities towards predicting the target variable $Y$. Note that by definition, the information gap $\Delta_p$ only depends on the joint distribution $p$, i.e., the multimodal prediction problem itself, and is independent of the modality encoders $g_T$ and $g_V$. Hence, it is a constant during the modality learning process. As we shall see shortly, the {\em information gap} will serve as a lower bound of the downstream prediction error if we seek to find features that admit a {\em zero modality gap}. From this perspective, the information gap is the \emph{price} we have to pay for using perfectly aligned features among different modalities. Thus, it well corresponds to the modality gap we are interested in. We can now state our theorem as follows.
\begin{restatable}{theorem}{mgap}
\label{thm:gap}
For a pair of modality encoders $g_T(\cdot)$ and $g_V(\cdot)$, if the multi-modal features $Z_T = g_T(X_T)$ and $Z_V = g_V(X_V)$ are perfectly aligned in the feature space, i.e., $Z_T = Z_V$, then $\inf_{h}\sE_p[\lce(h(Z_T, Z_V), Y)] - \inf_{h'}\sE_p[\lce(h'(X_T, X_V), Y)] \geq \Delta_p$.
\end{restatable}
\noindent\textbf{Remark}~~We discuss some of the implications of the above theorem. At a high level, Theorem~\ref{thm:gap} states that if the information gap $\Delta_p$ between the two modalities is large, then the optimal prediction error we can hope to achieve by using modality-aligned features is at least $\Delta_p$ larger than that we can achieve from the input modalities. In particular, when only one of the modalities contains predictive information w.r.t.\ the downstream target $Y$, enforcing perfect modality alignment could render the learned modality-aligned features $Z_T$ and $Z_V$ uninformative of $Y$, leading to a large downstream prediction error. Intuitively, such a phenomenon will happen because modality alignment enforces the aligned features to only contain predictive information present in both of the input modalities $X_T$ and $X_V$. 

In practice, because of the use of contrastive loss, due to the asymptotic behavior of it~\cite[Theorem 1]{wang2020understanding}, in the limit of infinity amount of data, the contrastive loss will force positive pairs to be perfectly aligned. In the context of multi-modal learning, this means that the assumption $Z_T = Z_V$ of Theorem~\ref{thm:gap} will hold. As a last note, we comment that the requirement of perfect alignment in Theorem~\ref{thm:gap} is not necessary: the lower bound could be extended when the features $Z_T$ and $Z_V$ are only approximately aligned.\footnote{For example, one could use a relax parameter $\epsilon$ to characterize the degree of modality alignment in the feature space. Then, we only need to replace the existing lower bound $\Delta_p$ with $\Delta_p- \epsilon$.} 

Due to space limit, we defer the proof of Theorem~\ref{thm:gap} to~\Cref{sec:app}. In fact, it can be readily seen from the proof in the appendix that we could relax the exact modality alignment condition in Theorem~\ref{thm:gap} even further. In other words, as long as there exists a bijection between $Z_T$ and $Z_V$, then the conditional mutual information satisfies $I(Z_V;Y\mid Z_T) = I(Z_T;Y\mid Z_V) = 0$, so the exact same lower bound in Theorem~\ref{thm:gap} will hold. 
\section{Method}
\label{sec:method}
\begin{figure*}[t]
  \centering
  \begin{subfigure}{0.45\linewidth}
    \centering
    \includegraphics[width=3.2in]{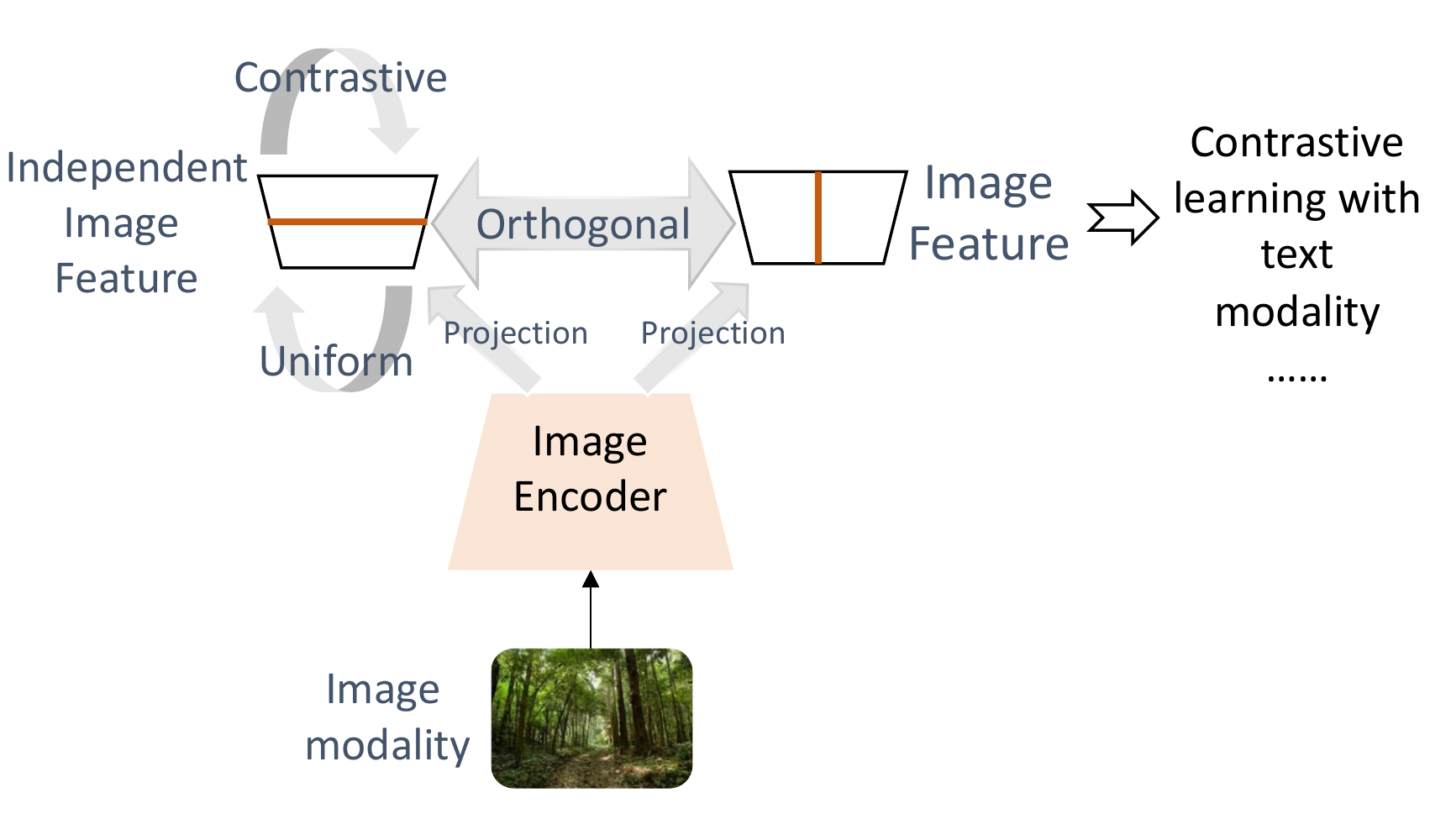}
    \vspace{-0.5cm}
    \caption{Building deep feature separation to preserve {\em modality-independent information}. Independent image features are enforced to be complimentary from original feature~(with orthogonal loss) and store meaningful information~(with contrastive and uniform losses).}
    \label{fig:sep}
  \end{subfigure}
\hspace{0.2cm}
  \begin{subfigure}{0.45\linewidth}
    \centering
    \includegraphics[width=2.5in]{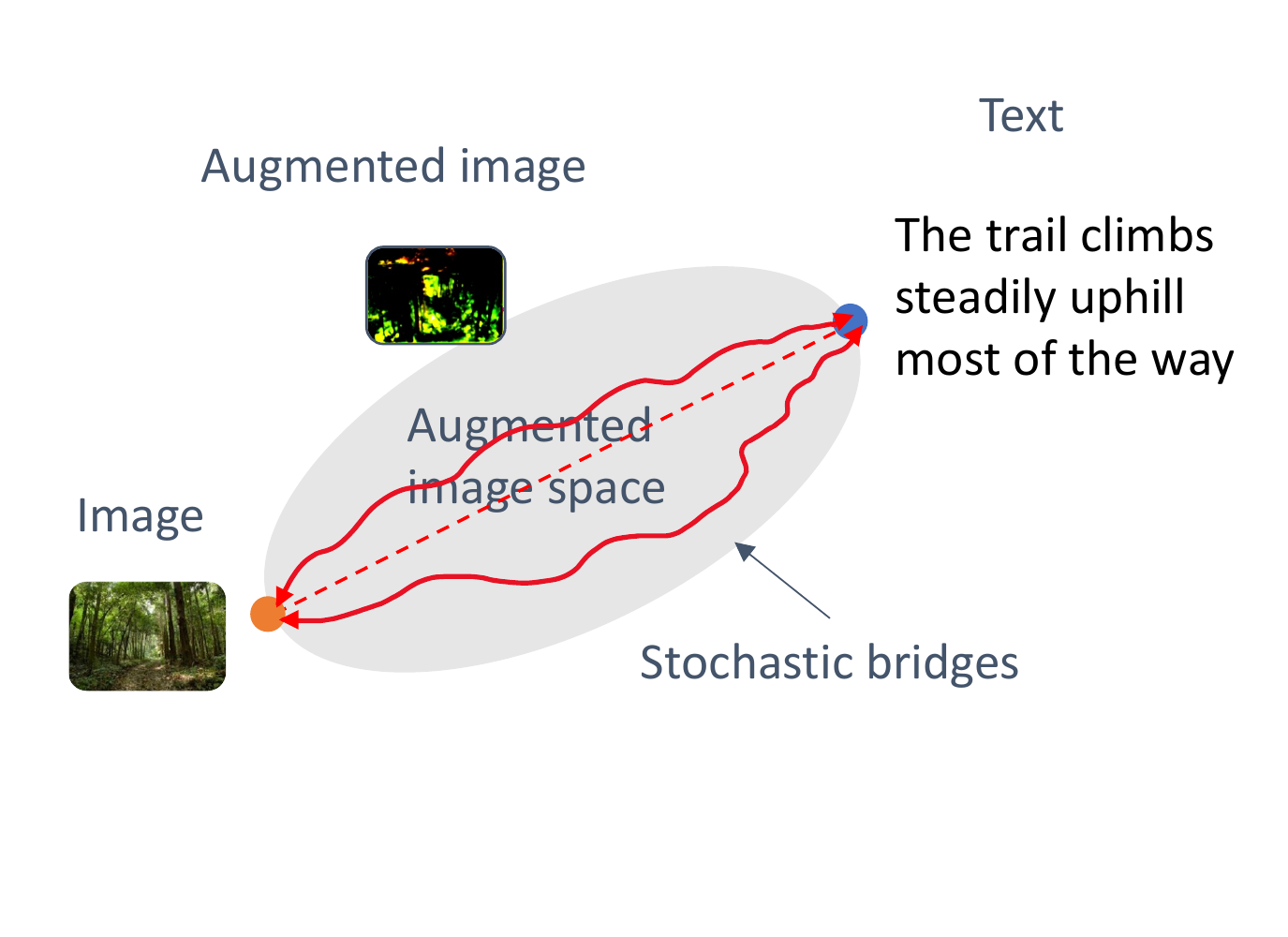}
    \vspace{-0.9cm}
    \caption{Building Brownian bridges between the image and text modalities to regularize inter-modality representations. Each red curve illustrates a stochastic bridge connecting an image-text pair; and the augmented images are enforced to stay on the path, guiding a cross-modality structure to connect the image and text modalities.}
    \label{fig:br}
  \end{subfigure}
  \vspace{0.3cm}
    \begin{subfigure}{0.85\linewidth}
    \centering
    \includegraphics[width=5in]{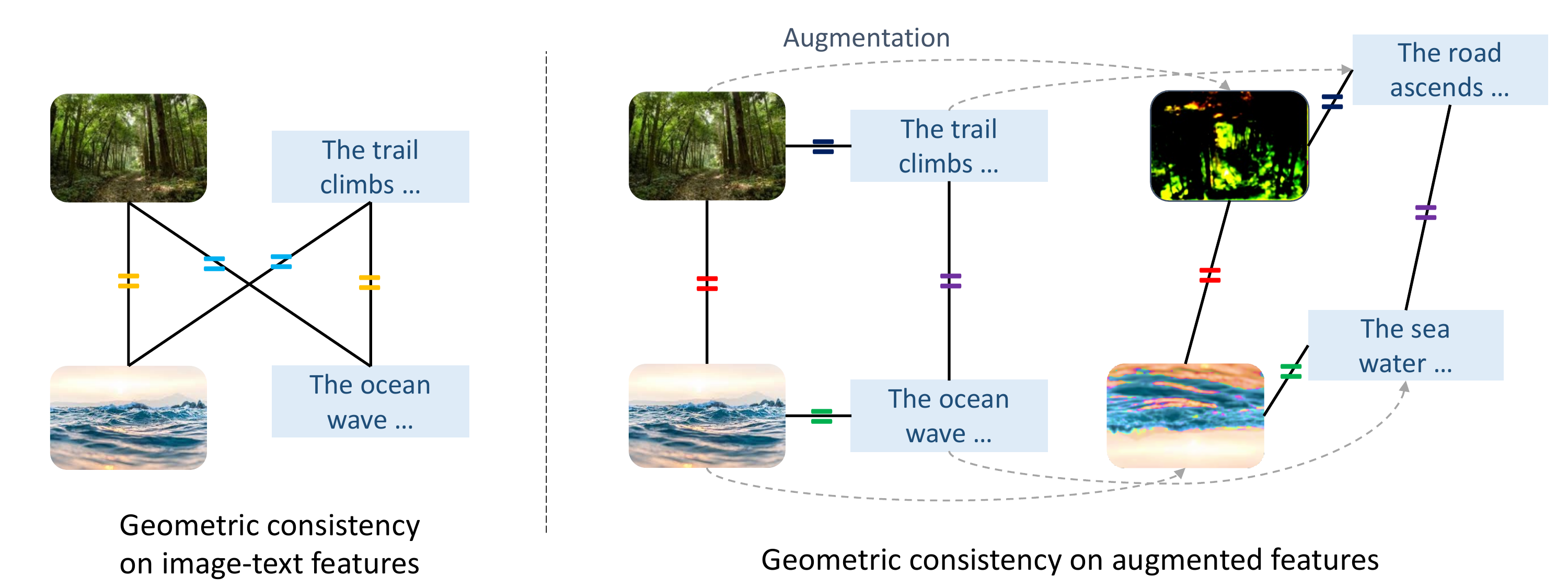}
    \caption{Building geometric consistency between features. Each solid line represents the distance between two features. Same colored $=$ signs indicate that the symmetry is encouraged, \ie the two distances are encouraged to be the same.}
    \label{fig:gc}
  \end{subfigure}
  \vspace{-0.5cm}
  \caption{Illustration of our three designed regularizer for constructing latent feature structure.}
  \label{fig:three_regularizer}
  \vspace{-0.5cm}
\end{figure*}

Motivated by Theorem~\ref{thm:gap}, instead of seeking exact modality matching, in this section we propose to construct meaningful {\em latent modality structures}. They can play an important role in learning generalizable multi-modal representations by preventing pure modality alignment. In the following, we propose three designs from different perspectives to construct the latent modality structures, by considering variations in intra- and inter-modalities. We visualize these designs in~\figref{fig:three_regularizer}. 
We first introduce the basic contrastive learning framework that we develop our methods on.
Following previous work~\cite{jiali2021codis,radford2021clip}, we adopt the multi-modal training framework with contrastive loss, which uses both cross-modal and in-modal contrastive loss, \ie,$\mathcal{L}_{\text{Con}} = \frac{1}{4}(\mathcal{L}_{\text{V2T}} + \mathcal{L}_{\text{T2V}} + \mathcal{L}_{\text{V2V}} + \mathcal{L}_{\text{T2T}} )$ with:
\begin{align*}
    \mathcal{L}_{\text{V2T}}=&-\frac{1}{N}\sum_{j=1}^N \log\frac{e^{\product{z_{V_j},z_{T_j}}/\tau}}{\sum_{k=1}^N e^{\product{z_{V_j},z_{T_k}}/\tau}}\\
    \mathcal{L}_{\text{V2V}}=&-\frac{1}{N}\sum_{j=1}^{N} \log\frac{e^{\product{z_{V_j},z^\text{a}_{V_j}}/\tau}}{\sum_{k=1}^N e^{\product{z_{V_j},z_{V_k}}/\tau}}\\
\end{align*}
where $N$ denotes the batch size; $z_{V_j}$ denote the feature of the $j$-th image in the mini-batch, with its augmentation $z^\text{a}_{V_j}$ and corresponding text feature $z_{T_j}$. The remaining losses~($\mathcal{L}_{\text{T2V}}$, $\mathcal{L}_\text{{T2T}}$) are defined in the same way by switching between text modality ($T$) and image modality ($V$).


\subsection{Intra-modality Regularization via Deep Feature Separation}
This subsection aims to construct intra-modality structures to regularize in-modality representations. Based on Theorem~\ref{thm:gap}, we first define two types of information, {\em modality-shared information} that is shared by all modalities, and {\em modality-independent information} that is modality-specific. Our motivation stems from our theoretical finding that exact modality matching is sub-optimal due to the loss of {\em modality-independent information}. To overcome this limitation, we propose to explicitly model the modality-independent information. We achieve this by applying the idea of feature separation~\cite{bousmalis2012domain} on multi-modal representation learning. Our basic construction is shown in Figure~\ref{fig:sep}. On top of the contrastive learning framework, we use additional projection layers to construct new features to store such information. We term these {\em independent features}, meaning that they contain modality-specific information independent of the other modality. We take extra constraints to ensure that a) independent features contain complementary information from the original features; and b) independent features are meaningful representations.

To ensure a), we constrain the features to be orthogonal to the original features by forcing their inner product to be small, \ie $\product{u,v}=0$. We define an orthogonal loss over minibatch optimization as follows:
\begin{equation*}
       \mathcal{L}_{\text{Ortho}} = \frac{1}{N}\sum_{j=1}^N(\product{z_{V_j}, z_{V_j}^{\text{i}}}^2 + \product{z_{T_j}, z_{T_j}^{\text{i}} }^2) 
\end{equation*}
where $z_{V_i}^{\text{i}}$ denote the independent feature of the $i^{th}$ image feature in the batch.

To avoid the degenerate case where the independent features are learned to be non-informative noises independent of the other modality, we further constrain that the independent features are informative. To this end, we adopt the contrastive loss and uniformity loss on the independent features, {\it i.e.}, we first adopt in-modality contrastive loss for independent text features and independent image features separately, \ie, $\mathcal{L}_{\text{Con}}^{\text{i}} =\mathcal{L}_{\text{V2V}}^{\text{i}} + \mathcal{L}_{\text{T2T}}^{\text{i}}$ with
\begin{align*}
    \mathcal{L}_{\text{V2V}}^{\text{i}}=&-\frac{1}{N}\sum_{j=1}^N \log\frac{e^{\product{z_{V_j}^{\text{i}},z_{V_j}^{\text{i~a}}}/\tau}}{\sum_{k=1}^N e^{\product{z_{V_j}^{\text{i}},z_{V_k}^{\text{i}}}/\tau}}~,
\end{align*}
and $\mathcal{L}_{\text{T2T}}^{\text{i}}$ is defined similarly. Then
we enhance the independent features with the uniformity loss~\cite{wang2020hypersphere} that maximizes the pairwise Gaussian potential~\cite{Cohn2006UniversallyOD,Bachman2019LearningRB}. Such a uniformity loss encourages the learned features to preserve maximal information: 
\vspace{-0.3cm}
\begin{align*}
        \mathcal{L}_{\text{Uni}}^{\text{i}} =& \log \frac{1}{N}\sum_{j=1}^N\sum_{k=1}^N G_t(z_{V_j}^{\text{i}},z_{V_k}^{\text{i}}) + G_t(z_{T_j}^{\text{i}},z_{T_k}^{\text{i}}),
\end{align*}
where $G_t(u,v) = e^{-t\|u-v\|^2}$ is the Gaussian potential kernel with $t=2$. In this way, we can preserve both {\em modality-shared information} and {\em modality-independent information}. 
Finally we obtain the total loss:
 $\mathcal{L}_{\text{Sep}} =\mathcal{L}_{\text{Ortho}}+ \mathcal{L}_{\text{Con}}^{\text{i}}+ \mathcal{L}_{\text{Uni}}^{\text{i}}$. 

\begin{figure*}[t]
  \centering
  \begin{subfigure}{0.45\linewidth}
    \centering
    \includegraphics[width=2.56in]{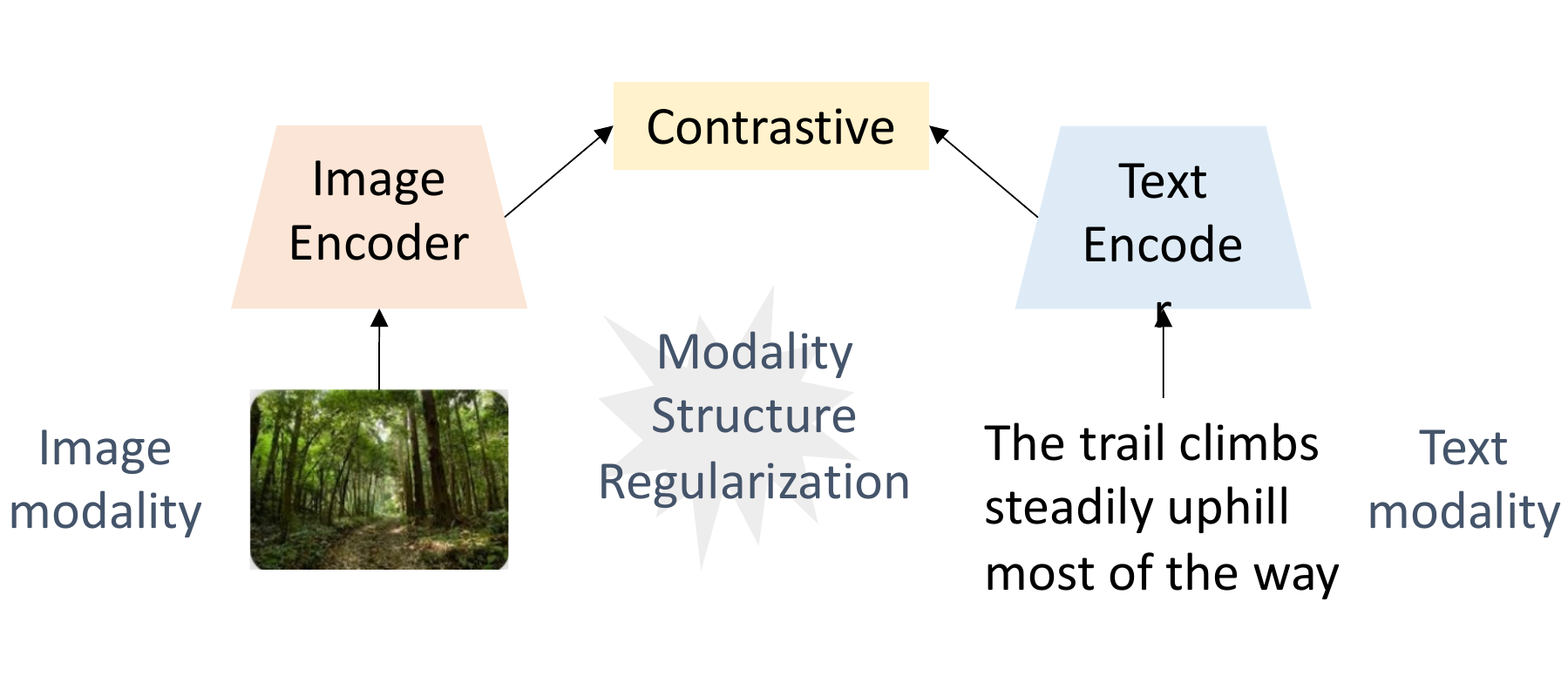}
    \vspace{-0.2cm}
    \caption{Two-tower-based models.}
    \label{fig:clip}
  \end{subfigure}
  \begin{subfigure}{0.45\linewidth}
    \centering
    \includegraphics[width=3.2in]{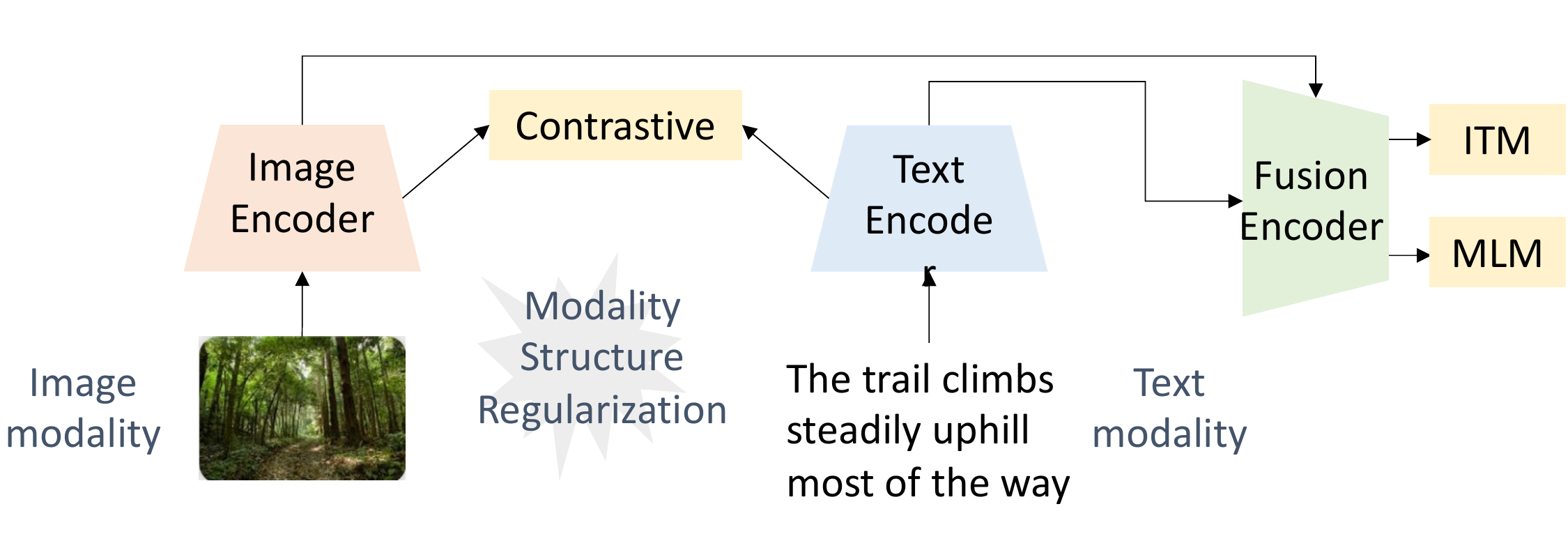}
    \vspace{-0.5cm}
    \caption{Fusion-based models.}
    \label{fig:albef}
  \end{subfigure}
  \vspace{-0.3cm}
  \caption{Illustration of two-tower-based models~(\eg CLIP) and fusion-based models (\eg ALBEF). Our latent modality regularization can be applied to both type of models at their feature level.}
  \label{fig:clip_albef}
  \vspace{-0.3cm}
\end{figure*}
\subsection{Inter-modality Regularization via Brownian Bridge}
Next, we consider regularizing inter-modality structures.
With the existence of modality gap, a natural idea is to constrain paired modality features in some subspace so that they are better separated from other feature pairs. To this end, we propose to construct a latent structure to explicitly guide the transition from the image modality to the associated text modality. Such a modality transition can be seamlessly modeled by the so-called Brownian bridge~\cite{Wang2022brownian,Mansuy2008AspectsOB}, which is a special type of Brownian motion with constraints that define stochastic paths (called bridges) between a pair of fixed starting and ending points (corresponding to the two modalities in our setting). Our basic construction is illustrated in Figure~\ref{fig:br}. 

To formulate this, given two random variables $(Z_V, Z_T)$ of image-text feature pairs, we denote the feature of augmented image as $Z_V^{\text{a}}$. We define a stochastic path such that $Z_V^{\text{a}}$ is constrained to stay on the path between $Z_V$ and $Z_T$. From the property of Brownian bridge, this endows a conditional Gaussian distribution of the form:
\vspace{-0.3cm}
\begin{equation}\label{eq:bridge}
p(Z_V^{\text{a}}|Z_V, Z_T) = \mathcal{N}(Z_V^{\text{a}}; \mu(Z_V, Z_T, t), t(1-t)\mathbf{I})  
\end{equation}
where $t\in[0, 1]$ is a hyperparameter, which can be randomly sampled at each time or fixed to a pre-defined value (we fix it to 0.25 in our experiments for simplicity); $\mu(Z_V, Z_T, t) \triangleq \frac{tZ_V + (1-t)Z_T}{\|tZ_V + (1-t)Z_T\|}$, and the normalizer is applied in order to constrain the mean to lie on the hyper-sphere feature space. Based on the maximal likelihood principle, to fit the model, we can simply align the $Z_V^{\text{a}}$ with the mean of the Brownian bridge in \eqref{eq:bridge}. When applying stochastic optimization, this ends up with optimizing the following objective at each time over a mini-batch:
\vspace{-0.1cm}
\begin{align*}
    \mathcal{L}_{\text{Br}} &= \frac{1}{N}\sum_{j=1}^N\|z_{V_j}^{\text{a}} - \mu(z_{V_j}, z_{T_j}, t)\|^2 \\
    &= \frac{1}{N}\sum_{j=1}^N\frac{t \product{z_{V_j}, z_{V_j}^{\text{a}}} + (1-t)\product{z_{T_j}, z_{V_j}^{\text{a}}}}{t^2+(1-t)^2 + 2t(1-t)\product{z_{V_j}, z_{T_j}}}
\end{align*}

\subsection{Intra-Inter Regularization via Geometric Consistency}
In the previous subsections, we consider either intra- or inter-modality structures between the two modalities. Is it possible to relate these two types of relationships together? In this subsection, we aim to design a general regularizer that considers both intra- and inter-modality structures. We achieve this goal by enforcing geometric symmetry within and between modality representations and their augmentations. Specifically, we generalize the idea in CyCLIP~\cite{Goel2022CyCLIPCC} so that it also includes geometric consistency for the augmented features, which is demonstrated in the experiments to achieve significant improvement. 

Specifically, we apply two types of geometric consistency losses that achieve symmetry in the following settings. 
First, we enforce geometric consistency among the original modality features, by optimizing the similarity between the mismatched image and text pairs, and the similarity between image pairs and text pairs. As shown in Figure~\ref{fig:gc}, we achieve this by encouraging the geometric consistency such that $\product{z_{V_1}, z_{T_2}}\sim\product{z_{V_2}, z_{T_1}}$ and $\product{z_{V_1}, z_{V_2}}\sim\product{z_{T_1}, z_{T_2}}$, where $a \sim b$ means $a$ is close to $b$ in some sense (defined below). 
We define the following geometric consistency objective over mini-batch: 
\vspace{-0.3cm}
\begin{align*}
    \mathcal{L}_\text{GC}=&\frac{1}{N}\sum_{j=1}^N\sum_{k= 1}^N[(\product{z_{V_j}, z_{T_k}}-\product{z_{V_k}, z_{T_j}})^2 \\
    +&(\product{z_{V_j}, z_{V_k}}-\product{z_{T_j}, z_{T_k}})^2]
\end{align*}
Second, we optimize the geometric consistency of augmented features. As shown in~\figref{fig:gc} we optimize geometric symmetry between feature pairs and augmented feature pairs in the text and image space.
The following objective is used to enforce this goal:
\begin{align*}
    &\mathcal{L}_\text{GC}^\text{a}=\frac{1}{N}\sum_{j=1}^N\sum_{k=1}^N[(\product{z_{V_j}, z_{V_k}}-\product{z_{V_j}^\text{a}, z_{V_k}^\text{a}})^2 \\
    \hspace{-0.3cm}+&(\product{z_{T_j}, z_{T_k}}-\product{z_{T_j}^\text{a}, z_{T_k}^\text{a}})^2] 
    +\frac{1}{N}\sum_{j=1}^N(\product{z_{V_j}, z_{T_j}}-\product{z_{V_j}^\text{a}, z_{T_j}^\text{a}})^2
\end{align*}
Overall, the total combination of geometric consistency loss can be written as: $ \mathcal{L}_\text{GC} +\mathcal{L}_\text{GC}^\text{a}$.

\paragraph{Final Loss}
We can now define a final loss by combining the standard contrastive loss with one or several of our proposed modality regularization losses. The effect of each regularization could be task-dependent, \ie certain task could benefit more from certain regularization, which we will show comprehensively in the next section. 

\section{Experiments}
\begin{table*}[t]
\small
    \centering
    \caption{Zero-shot TopK classification accuracy~(\%) on CIFAR10, CIFAR100 and ImageNet1K.}
    \vspace{-0.3cm}
    \label{tab:clip_zero}
    \setlength{\tabcolsep}{5pt}
    \resizebox{0.64\linewidth}{!}{
    \begin{tabular}{lccc|ccc|ccc}
    \specialrule{.15em}{.05em}{.05em}
      \multirow{2}{*}{Method} & \multicolumn{3}{c}{CIFAR10} & \multicolumn{3}{c}{CIFAR100}&\multicolumn{3}{c}{ImageNet1K}\\
        & Top1 & Top3 & Top5& Top1 & Top3 & Top5& Top1 & Top3 & Top5\\
      \hline
      CLIP~\cite{radford2021clip} &44.95 & 72.58 & 88.3 &15.05 &29.51 & 37.53 & 16.72 & 28.61 & 34.38\\
      CyCLIP~\cite{Goel2022CyCLIPCC} & 43.22 &71.43 & 83.22 & 15.09 & 27.39 & 34.35 & 17.77 & 30.06 & 36.20\\
      \cmidrule{2-10}
      OURS$_\text{Sep}$ & 46.61 & \bf81.21 & \bf92.44 & 19.37 & 36.66 & 46.26 & 20.21 & 33.25 & 39.60\\
      OURS$_\text{Br}$ & 43.15 & 72.77 &86.72 & 14.22 & 26.46 & 33.28 & \bf20.45 & \bf33.56 & 39.28 \\
      OURS$_\text{GC}$ & \bf56.36 & 80.47 & 90.27 & \bf22.70 & \bf41.66 & \bf51.78 &20.25 & 33.50 & \bf39.91\\
      
    \specialrule{.15em}{.05em}{.05em}     
    \end{tabular}
    }
    \vspace{-0.2cm}
\end{table*}
\begin{table*}[t]
\small
    \centering
    \caption{Zero-shot TopK classification accuracy~(\%) on Natural Distribution Shifts.}
    \vspace{-0.3cm}
    \label{tab:clip_shift}
    \setlength{\tabcolsep}{5pt}
    \resizebox{0.85\linewidth}{!}{
    \begin{tabular}{lccc|ccc|ccc|ccc}
    \specialrule{.15em}{.05em}{.05em}
      \multirow{2}{*}{Method} & \multicolumn{3}{c}{ImageNetV2} & \multicolumn{3}{c}{ImageNetSketch}&\multicolumn{3}{c}{ImageNet-A}&\multicolumn{3}{c}{ImageNet-R}\\
        & Top1 & Top3 & Top5& Top1 & Top3 & Top5& Top1 & Top3 & Top5& Top1 & Top3 & Top5\\
      \hline
      CLIP~\cite{radford2021clip} & 14.11 & 25.76 & 31.80 & 8.61 & 16.47 & 21.13 & 2.81 & 7.31 & 11.32 & 19.07 & 31.99 & 39.03\\
      CyCLIP~\cite{Goel2022CyCLIPCC} & 15.25 & 26.59 & 32.15 & 8.30 & 16.18 & 20.77 & 3.27 & 8.45 & 13.07 &19.85 &33.35 & 40.35\\
      \cmidrule{2-13}
      OURS$_\text{Sep}$ & 16.78 & 28.97 & 35.68 & 9.22 & 17.86 & 23.00 & 3.45 & 9.88 & 15.81 & 22.06 & 35.65 & 43.01\\
      OURS$_\text{Br}$ & 17.02 & 29.39 & 35.53 & 10.34 & 18.39 & 23.05 & 3.01 & 7.50 & 11.45 & 20.40 & 32.43 & 38.45\\
      OURS$_\text{GC}$ & \bf17.37 & \bf29.84 & \bf36.65 & \bf10.90 & \bf20.77 & \bf26.11 & \bf3.87 & \bf11.36 & \bf16.76 & \bf23.85 & \bf37.90 & \bf45.03\\
      
    \specialrule{.15em}{.05em}{.05em}     
    \end{tabular}
    }
    \vspace{-0.2cm}
\end{table*}
\begin{table*}[!t]
\small
    \centering
    \caption{Linear probing Top1 classification accuracy~(\%) on visual benchmarks.}
    \vspace{-0.3cm}
    \label{tab:clip_linear}
    \setlength{\tabcolsep}{5pt}
    \resizebox{\linewidth}{!}{
    \begin{tabular}{lcccccccccccccc|c}
    \specialrule{.15em}{.05em}{.05em}
        \rule{0pt}{6ex} 
        & \rot{Caltech101} & \rot{SVHN} &\rot{STL10} & \rot{CIFAR10} & \rot{CIFAR100} & \rot{DTD} & \rot{FGVCAircraft}  & \rot{OxfordPets} & \rot{SST2} & \rot{Food101} & \rot{GTSRB} & \rot{StanfordCars} & \rot{Flowers102} &  \rot{ImageNet1K} & \rot{Average}\\
      \hline
      CLIP~\cite{radford2021clip} & 78.57 & 57.07 &87.22& 79.74 & 56.36 & 59.84 & 37.17 & 59.66 & 53.98 & 58.11 & 74.21 & 23.96 & 76.66 & 52.10 & 61.05\\
      CyCLIP~\cite{Goel2022CyCLIPCC} & 77.86 & 54.29 &87.61& 77.53 & 54.23 & 58.19 & 33.00 & 62.63 & 54.81 &  60.82 & 72.95 & 23.36 & 72.89 & 52.83 & 60.14\\
      \cmidrule{2-15}
      OURS$_\text{Sep}$ & \bf84.45 & \bf69.82 & 90.96 & \bf81.51 & \bf61.19 & \bf67.50 & \bf41.70 & \bf67.16 & 54.26 & \bf 63.08 & \bf 82.35 & \bf31.76 & \bf81.69 & \bf56.73 & \bf66.73\\
      OURS$_\text{Br}$ & 82.18 & 57.46 & 90.69 & 79.42 & 57.72 & 64.84 & 34.74 & 65.71 & 54.04 & 60.52 & 73.61 & 26.50 & 78.44 & 53.87 &62.84\\
      OURS$_\text{GC}$ &  83.23 & 63.58 & \bf91.31 &80.92 & 58.89 & 65.43 & 34.83 & 64.51 & \bf55.19 & 60.80 & 76.84 & 26.95 & 78.76 & 54.96 &  64.01\\
      
    \specialrule{.15em}{.05em}{.05em}     
    \end{tabular}
    }
    \vspace{-0.5cm}
\end{table*}
Our proposed methods are general purposed. Thus, we choose to evaluate them with two popular multi-modal representation frameworks: the two-tower based models ({\it e.g}, CLIP) and the fusion based models ({\it e.g.}, ALBEF), as illustrated  in~\figref{fig:clip_albef}. Note that in CLIP, text inputs are augmented with EDA~\cite{wei-zou-2019-eda}, and image inputs are augmented with random augmentation such as flipping and cropping. In ALBEF, augmented features are obtained with additional momentum encoders.

\subsection{Two-Tower-based Models}

\label{sec:exp_clip}
For this set of experiments, we adopt the CLIP-based models, where two separate encoders are trained to align features from the image and text modalities. To regularize latent modality structures, our regularization losses are separately applied along with the standard contrastive loss for pre-training\footnote{We will combine all the proposed regularizers for evaluation in experiments with the fusion-based models.}. We then evaluate on standard benchmarks.

\myparagraph{Setup:}
Our CLIP model adopts ResNet-50~\cite{He2016resnet} as the image encoder and BERT~\cite{devlin2018bert} as the text encoder. 
We adopt the official code from CyCLIP to incorporate our regularizations, as well as to reproduce the baselines. Our reproduced CLIP results are consistent with the recent works~\cite{mu2021slip,gao2021clipadp}, although they are slightly lower than reported in the original CLIP paper. The reason could be that the number of GPUs we use is different and we provide details in~\cref{sec:app_implementation_clip}. For both baselines, we can reproduce better performance on linear probing but slightly under-perform on zero-shot transfer, which we consider reasonable. Note that all methods are under the same codebase and same hyper-parameter setting, thus the comparisons are fair.

\myparagraph{Pre-training:} We follow the protocol of previous works to pre-train the model with the CC3M~\cite{sharma2018conceptual} dataset, which contains 3M unique images and 4M image-text pairs. 




\subsubsection{Zero-Shot Transfer Learning Evaluation}
We perform zero-shot transfer on standard image classification tasks, with the CIFAR10, CIFAR100~\cite{Krizhevsky2009cifar} and ImageNet1K~\cite{Russakovsky2015ImageNetLS} datasets. We use the standard evaluation strategy of prompt engineering. For each dataset, we construct the text prompts using the name of the class, \eg "a photo of the \texttt{[class name]}". For each class, we obtain the normalized class text embedding. During the evaluation, the class with the highest similarity score to the image embedding is predicted to be the label. Following previous works, we report Top-K classification accuracy with $K=1, 3, 5$. 

As shown in ~\tabref{tab:clip_zero}, our method significantly outperforms CLIP and CyCLIP on all three datasets, demonstrating the importance of latent modality structures. It is also interesting to see the differences our three regularizers perform in different datasets, {\it i.e.}, the feature-separation regularizer performs best in CIFAR10, while Brownian bridge regularizer performs best on ImageNet1K, and geometry consistency regularizer performs the best on CIFAR100.

\subsubsection{Natural Distribution Shift Evaluation}
We further evaluate variants~\cite{Recht2019imagenetv2,wang2019imagenetsketch,Hendrycks2021imageneta,Hendrycks2021imagenetr} of ImageNet1K dataset with shifted distributions. These datasets contain sketches, cartoons and adversarial generated images.
As shown in~\tabref{tab:clip_shift}, all methods suffer from performance degradation on natural distribution shift benchmarks compared to the performance on original ImageNet1K in~\tabref{tab:clip_zero}. Nevertheless, our method consistently outperforms the baselines on all benchmarks. In contrast to the other experiments, our geometric consistency regularization performs the best on all the benchmarks. 

\subsubsection{Linear Probing Evaluation}
We demonstrate better latent structure can also benefit downstream tasks with in-domain supervision. We evaluate this on linear probing tasks by fitting a linear classifier with in-domain supervision using the learned visual encoder. In total, we evaluate on 14 standard benchmarks~\cite{Krizhevsky2009cifar,Russakovsky2015ImageNetLS,li2006caltech101,Netzer2011svhn,Coates2011stl10,cimpoi14dtd,maji13aircraft,parkhi12pets,socher2013sst,bossard14food,Houben2013gtsrb,KrauseStarkDengFei-Fei_3DRR2013cars,Nilsback2008flowers}.
As shown in ~\tabref{tab:clip_linear}, all our methods outperform the baselines on all benchmarks by large margins. Remarkably, our deep feature separation regularization performs particularly well on this task. We believe this is partially because such regularization can learn to preserve more information that could be useful with extra in-domain supervision.
\subsection{Fusion-based Models}
\label{sec:exp_albef}
\begin{table}[t]
\small
    \centering
    \caption{Downstream tasks performance on fusion-based models.}
    \vspace{-0.3cm}
    \label{tab:albef_down}
    \setlength{\tabcolsep}{4pt}
    \resizebox{\linewidth}{!}{
    \begin{tabular}{lcccccc}
    \specialrule{.15em}{.05em}{.05em}
      \multirow{2}{*}{Method} & \multicolumn{2}{c}{VQA} & \multicolumn{2}{c}{NLVR$^2$} & \multicolumn{2}{c}{SNLI-VE}\\
      & test-dev & test-std & dev & test-P & val & test\\
      \hline
      ImageBERT~\cite{Li2019VisualBERTAS} & 70.80 & 71.00 & 67.40 & 67.00 & - & -\\
      LXMERT~\cite{Tan2019LXMERTLC} & 72.42 & 72.54 & 74.90 & 74.50 & - & - \\
      12-in-1~\cite{Lu202012in1MV} & 73.15 & - &- 78.87 & - & 76.95 \\
      UNITER~\cite{chen2020uniter} & 72.70 & 72.91 & 77.81 & 77.85 & 78.59 & 78.28 \\
      OSCAR~\cite{li2020oscar} & 73.16 & 73.44 & 78.07 & 78.36 & - & - \\
      VILLA~\cite{Gan2020LargeScaleAT}& 73.59 & 73.67 & 78.39 & 79.30 & 79.47 & 79.03\\
      ViLT~\cite{Kim2021ViLTVT} & 70.94 &- & 75.24 & 76.21 & - & -\\
      ViCHA~\cite{Shukor2022EfficientVP} & 73.55 & - & 78.14 & 77.00 & 79.20 & 78.65 \\
      ALBEF~\cite{li2021albef} &73.38&73.52&78.36&79.54 & 79.69 & 79.91\\
      CODIS~\cite{jiali2021codis}& 73.15&73.29 &78.58 & \bf79.92 & 79.45 & 80.13\\
      \cmidrule{2-7}
      OURS$_\text{Sep}$ & 73.52 & 73.59 & \bf79.05 & 79.76 & \bf79.95 & 79.61\\
      OURS$_\text{Br}$ & \bf74.26&\bf74.36 & 78.70 & 79.36 & 79.86 & 79.95\\
      OURS$_\text{GC}$ & 73.90 & 73.87 & 78.96 & 79.53 & 79.82 & \bf80.16\\
      
    \specialrule{.15em}{.05em}{.05em}     
    \end{tabular}
    }
\end{table}
We next test our methods on fusion-based models. We adopt the ALBEF~\cite{li2021albef} framework, where a fusion encoder is applied to fuse the modality as shown in~\figref{fig:albef}. Such fusion-based models are known to be more powerful in learning inter-model interaction compared to simple two-tower-based models. Thus, we evaluate our methods on various vision-language downstream tasks including VQA~\cite{goyal2017vqa}, NLVR$^2$~\cite{suhr2019nlvr2}, SNLI-VE~\cite{bowman2015snli}. Here we incorporate all three regularizations for these tasks. We additionally provide ablation study on smaller scale experiments.
\paragraph{Setup}
We use ViT-B/16 as our vision encoder and 12-layer BERT$_{base}$ as the text encoder. Note the first 6 layers of BERT$_{base}$ are used purely as the text encoder and the remaining are used as fusion encoder. We reproduced ALBEF and CODIS results for fair comparisons. All experiments we run are under the same codebase and hyper-parameter settings. The details are included in~\cref{sec:app_implementation_albef}.

\myparagraph{Pre-training:}
We follow the previous experiments protocols~\cite{li2021albef,jiali2021codis} using a union of four datasets for pre-training, which include Conceptual Captions~(CC3M)\cite{sharma2018conceptual}, Visual Genome~(VG)\cite{kris2017vg}, SBU Captions\cite{ordonez2011im2text} and COCO\cite{lin2014coco}, constituting 4M unique images and 5M image-text pairs.

\subsubsection{Vision-Language Tasks Evaluation}

\myparagraph{Visual Question Answering (VQA):} 
We fine-tune and evaluate our pre-trained model on VQA v2.0. Following~\cite{cho2021unifying,li2021albef,jiali2021codis}, we consider VQA as a generation task. During fine-tuning, we apply 6-layer transformer-based decoder to generate the answer. 
We fine-tune on the training set and evaluate on the test-dev and test-std set. The results are presented in Table~\ref{tab:albef_down}. Consistently, our method performs the best and achieves a 1\% improvement on both the test-dev and test-std sets.

\myparagraph{Natural Language for Visual Reasoning (NLVR$^2$):} 
We use the NLVR$^2$ dataset, which contains 100K texts paired with web images. To enable our model to reason over two images, we follow~\cite{li2021albef} to extend the fusion encoder with an MLP prediction head and perform additional pre-training of one epoch to prepare the fusion encoder on text-assignment task. As shown in Table~\ref{tab:albef_down}, our method achieves an improvement of 2\% on the dev set and matches the performance of SOTA on the test-P set.

\myparagraph{Visual Entailment (VE):} 
We follow~\cite{li2021albef,chen2020uniter} and consider this as a classification problem with three classes~(entailment, neutral, contradictory). Thus, we adopt an MLP prediction head on top of the fusion encoder. Again, our method is comparable to the baselines on the val set and outperforms all baselines on the test set.



We provide additional results including analysis and visualization of constructing latent structures, visualization of experimental results, as well as ablation studies in~\Cref{sec:app_ablation}.


\section{Conclusion}
\label{sec:conc}
In this paper, we investigate the latent modality structures in multi-modal representation learning. We analyze and examine the modality gap in the latent feature space and reveal that reducing modality gap to zero does not always lead to better performance. Instead we advocate that more meaningful latent features structures will benefit the downstream applications. Thus we design three regularization methods to construct meaningful latent structures. We propose to use 1) deep feature separation loss 2) brownian bridge loss 3) geometric consistency loss to improve the latent features from different perspectives. Extensive experiments on multiple vision-language tasks including image classification, linear probing, visual question answering, visual reasoning, visual entailment confirm the effectiveness and the generalizability of our proposed approach on popular contrastive representation learning frameworks.

{\small
\bibliographystyle{ieee_fullname}
\bibliography{egbib}
}

\clearpage
\appendix
\section{Proof of Theorem~\ref{thm:gap}}
\label{sec:app}
To ease the reading, we first restate Theorem~\ref{thm:gap} and then provide the proof.
\mgap*
\begin{proof}[Proof of Theorem~\ref{thm:gap}]
Consider the joint mutual information $I(Z_T, Z_V; Y)$. By the chain rule, we have the following decompositions:
\begin{align*}
    I(Z_T, Z_V; Y) &= I(Z_T; Y) + I(Z_V; Y\mid Z_T) \\
                   &= I(Z_V; Y) + I(Z_T; Y\mid Z_V).
\end{align*}    
However, since $Z_T$ and $Z_V$ are perfectly aligned, $I(Z_V; Y\mid Z_T) = I(Z_T; Y\mid Z_V) = 0$, which means $I(Z_T, Z_V; Y) = I(Z_V; Y) = I(Z_T;Y)$. On the other hand, by the celebrated data-processing inequality, we know that
\begin{equation*}
    I(Z_T; Y) \leq I(X_T; Y), \quad I(Z_V;Y)\leq I(X_V; Y).
\end{equation*}
Hence, the following chain of inequalities holds:
\begin{align*}
    I(Z_T, Z_V; Y) &= \min\{I(Z_T; Y), I(Z_V; Y)\} \\
                   &\leq \min\{I(X_T; Y), I(X_V; Y)\} \\
                   &\leq \max\{I(X_T; Y), I(X_V; Y)\} \\
                   &\leq I(X_T, X_V; Y),
\end{align*}
where the last inequality follows from the fact that the joint mutual information $I(X_T, X_V; Y)$ is at least as large as any one of $I(X_T;Y)$ and $I(X_V;Y)$. Therefore, due to the variational form of the conditional entropy, we have
\begin{align*}
& \inf_{h}\sE_p[\lce(h(Z_T, Z_V), Y)] - \inf_{h'}\sE_p[\lce(h'(X_T, X_V), Y)] \\
=&~ H(Y\mid Z_T, Z_V) - H(Y\mid X_T, X_V) \\
    =&~ I(X_T, X_V; Y) - I(Z_T, Z_V; Y) \\
    \geq&~ \max\{I(X_T; Y), I(X_V; Y)\} - \min\{I(X_T; Y), I(X_V; Y)\} \\
    =&~ \Delta_p\qedhere.
\end{align*}
\end{proof}

\begin{table}[t]
\small
    \centering
    \caption{Downstream tasks performance on fusion-based models.}
    \vspace{-0.3cm}
    \label{tab:albef_app}
    \setlength{\tabcolsep}{4pt}
    \resizebox{\linewidth}{!}{
    \begin{tabular}{lcccccc}
    \specialrule{.15em}{.05em}{.05em}
      \multirow{2}{*}{Method} & \multicolumn{2}{c}{VQA} & \multicolumn{2}{c}{NLVR$^2$} & \multicolumn{2}{c}{SNLI-VE}\\
      & test-dev & test-std & dev & test-P & val & test\\
      \hline
      ALBEF~\cite{li2021albef} &73.38&73.52&78.36&79.54 & 79.69 & 79.91\\
      CODIS~\cite{jiali2021codis}& 73.15&73.29 &78.58 & \bf79.92 & 79.45 & 80.13\\
      \cmidrule{2-7}
      OURS$_\text{All}$&  74.12 & 74.16 & \bf 80.18 & 79.80 &79.62& \bf80.23\\
      OURS$_\text{Sep}$ & 73.52 & 73.59 & 79.05 & 79.76 & \bf79.95 & 79.61\\
      OURS$_\text{Br}$ & \bf74.26&\bf74.36 & 78.70 & 79.36 & 79.86 & 79.95\\
      OURS$_\text{GC}$ & 73.90 & 73.87 & 78.96 & 79.53 & 79.82 & 80.16\\
      
    \specialrule{.15em}{.05em}{.05em}     
    \end{tabular}
    }
\end{table}

\begin{table}[t]
\small
    \centering
    \caption{Ablation study on zero-shot image-text retrieval performance on Flickr30K with model pre-trained on COCO.}
    \vspace{-0.2cm}
    \label{tab:albef_zero}
    \setlength{\tabcolsep}{5pt}
    \resizebox{\linewidth}{!}{
    \begin{tabular}{lcccccc}
    \specialrule{.15em}{.05em}{.05em}
      \multirow{2}{*}{Method}
      &\multicolumn{3}{c}{Text Retrieval} & \multicolumn{3}{c}{Image Retrieval}\\
      & R@1 & R@5 & R@10 & R@1 & R@5 & R@10 \\
      \hline
      ALBEF~\cite{li2021albef} & 58.4 & 83.2 & 89.5 & 44.5 &69.8 &78.0\\
      CODIS~\cite{jiali2021codis} & 62.7 & 87.0 & 92.3 & 49.0 & 74.1 & 82.9\\
      \cmidrule{2-7}
      OURS$_\text{Sep}$ & \bf66.0 & \bf88.2 & \bf93.9 & 50.4 & 76.2 & 83.7\\
      OURS$_\text{Br}$ & 65.4 & 88.1 & 93.1 & \bf50.8 & \bf77.1 & \bf84.4 \\
      OURS$_\text{GC}$ & 64.3 & 87.5 & 92.3 & 50.5 & 75.9 & 83.3\\
      
    \specialrule{.15em}{.05em}{.05em}     
    \end{tabular}
    }
    \vspace{-0.4cm}
\end{table}

\section{Additional Results}
\label{sec:app_ablation}
\subsection{Two-tower-based models}
\myparagraph{Visualization of constructing latent structures}
To better understand the effect of constructing latent modality structures, we visualize the effect of our method on the latent space in~\figref{fig:app_feature}. Note that all our methods achieve performance gain regardless of size of the modality gap, which complys with Section~\ref{sec:prelim} and Theorem~\ref{thm:gap}.

\myparagraph{Visualization of experimental results} To better demonstrate the effectiveness of our proposed methods, we visualize our experimental results on two-tower-based framework~(\eg CLIP) in~\figref{fig:app_bar} and~\figref{fig:app_radar}. Our methods show significant improvement on most of the tasks.

\subsection{Fusion-based models}
\myparagraph{Additional experimental results} 
We provide additional results on using all three regularizers. The results are shown in~\tabref{tab:albef_app}. While using all the regularizations together leads to performance gain, all our regularization methods improve the performance as well when used individually.

We evaluate zero-shot image-text retrieval on smaller scale experiments by pertaining on COCO and evaluate on Flickr30~\cite{Young2014flickr}. As shown in~\tabref{tab:albef_zero}, results indicate that all three regularizations improve the performance, while text retrieval benefits most from deep feature separation regularization and image retrieval task benefits most from Brownian bridge regularization.

\section{Implementation Details}
\label{sec:app_implementation}
\subsection{Two-tower based models.}
\label{sec:app_implementation_clip}
We follow the same code base and hyper-parameters setting as CyCLIP~\cite{Goel2022CyCLIPCC} except for number of GPUs. We train the model from scratch on 64 NVIDIA A100 GPUs and train for 64 epochs. Our batch size is 128 and feature dimension is 1024. We use an initial learning rate of $5e^{-4}$ with cosine scheduling. We warm-up the model for 10000 steps. We evaluate the model trained to the last epoch for our method.
\subsection{Fusion based models.}
\label{sec:app_implementation_albef}
We follow the codebase and hyper-parameter setting as~\cite{li2021albef,jiali2021codis} except for number of GPUs. We train all the models on 16 NVIDIA A100 GPUs. During the pre-train stage, we train with the pre-training tasks for 30 epochs. AdamW~\cite{Loshchilov2019DecoupledWD} optimizer is used along with weight decay of 0.02, batch size of 512, learning rate initially of 1e-5. We warm up the learning rate to 1e-4 after 1000 iterations and follow the cosine decay. The input size for pre-training task is 256 and the input sizes for downstream tasks are 384. 

\myparagraph{Reproducibility}
We follow the standard practice to fix the random seed to ensure that all our results are reproducible. All the source code will be made public upon acceptance of the paper.

\begin{figure*}[!t]
    \centering
    \includegraphics[width=\linewidth]{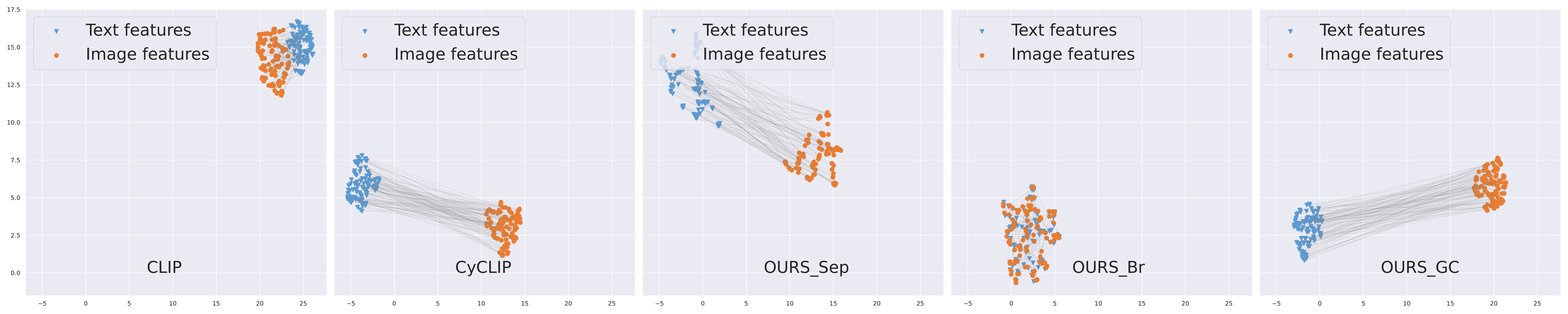}
    \vspace{-0.6cm}
    \caption{Visualization of constructing latent modality structures. Each line connects the positive image-text feature pair. }
    \vspace{-0.4cm}
    \label{fig:app_feature}
\end{figure*}

\begin{figure*}[b]
\begin{subfigure}{\linewidth}
    \centering
    \includegraphics[width=\linewidth]{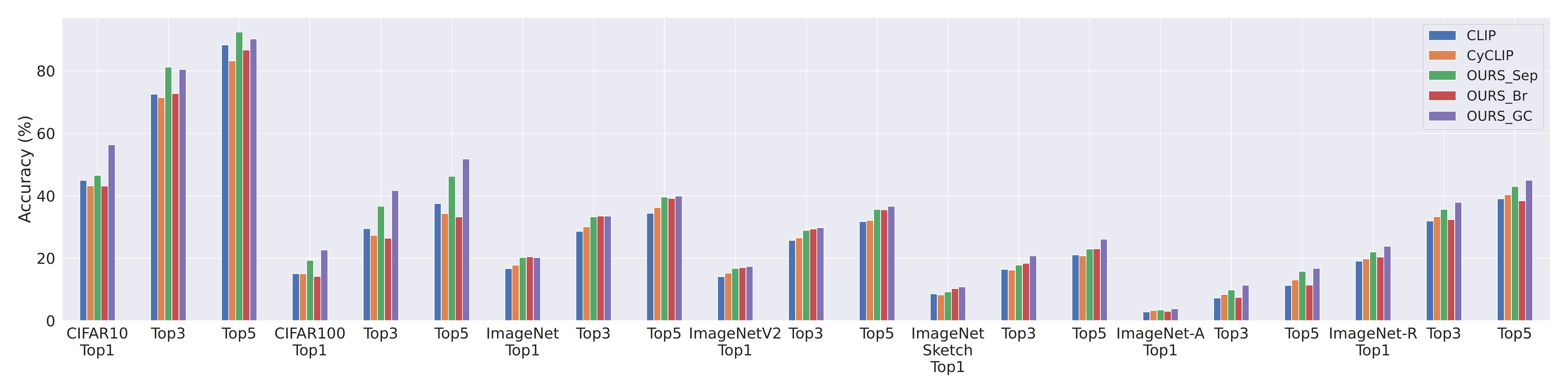}
    \vspace{-0.6cm}
    \caption{Zero-shot transfer performance.}
    \label{fig:bar_zero}
\end{subfigure}
\begin{subfigure}{\linewidth}
    \centering
    \includegraphics[width=\linewidth]{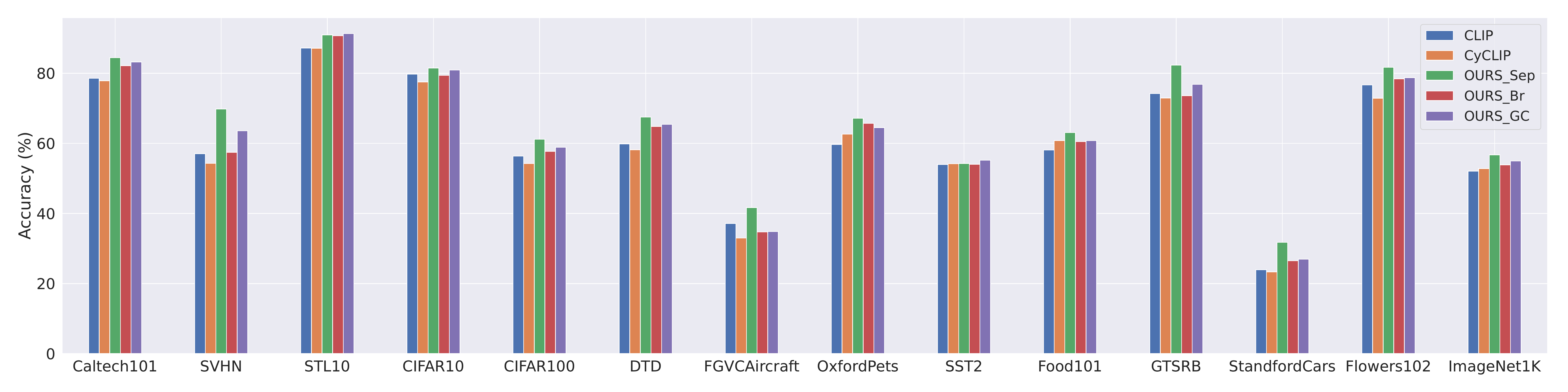}
    \vspace{-0.6cm}
    \caption{Linear-probing performance.}
    \label{fig:bar_linear}
\end{subfigure}
\vspace{-0.5cm}
\caption{Visualization of Two-tower-based methods~(\eg CLIP) performance. Each color represents a different approach.}
\label{fig:app_bar}
\end{figure*}

\begin{figure*}[b]
  \centering
  \begin{subfigure}{0.45\linewidth}
    \centering
    \includegraphics[width=3in]{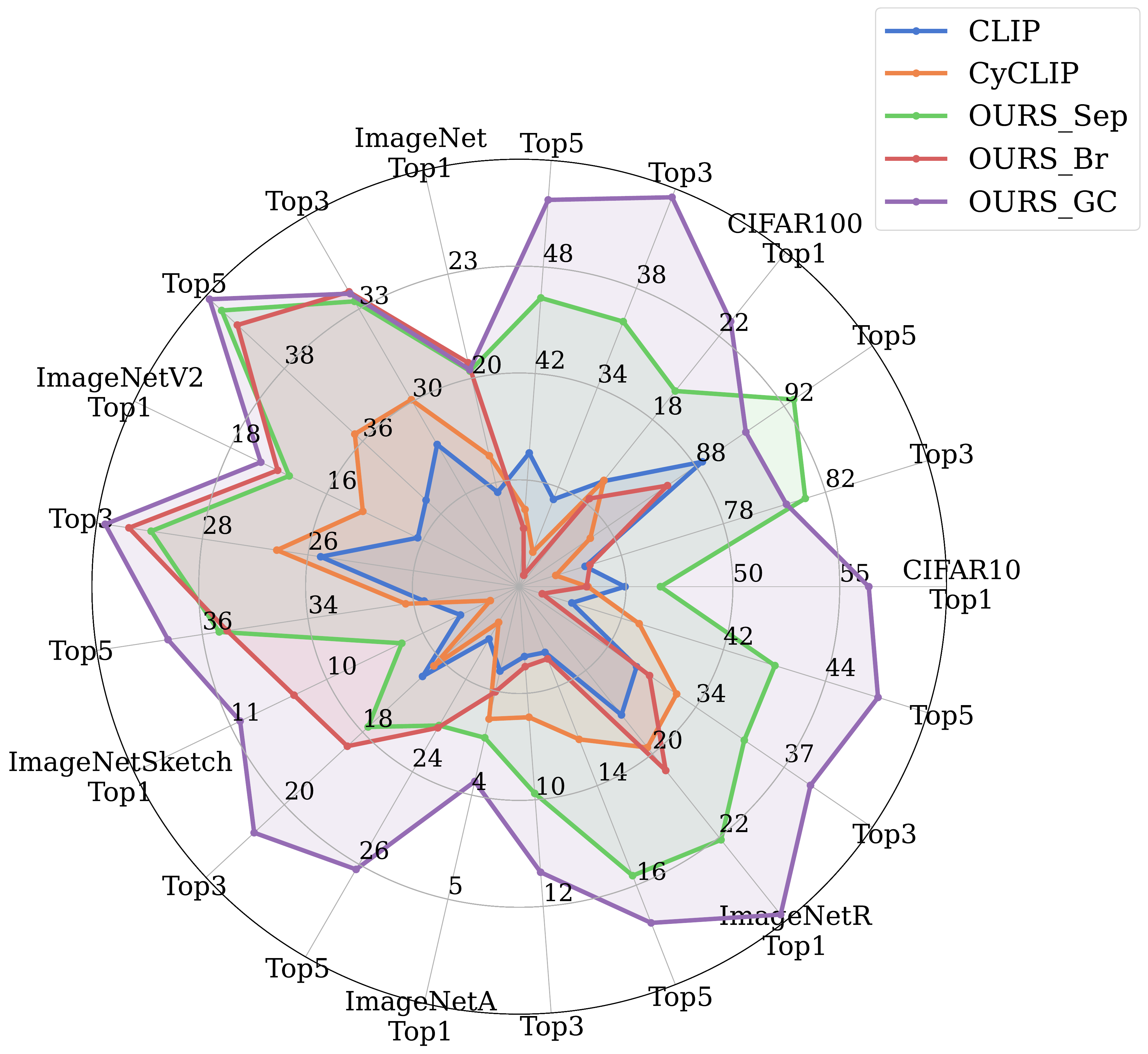}
    \caption{Zero-shot Performance.}
    \label{fig:clip}
  \end{subfigure}
  \begin{subfigure}{0.45\linewidth}
    \centering
    \includegraphics[width=3in]{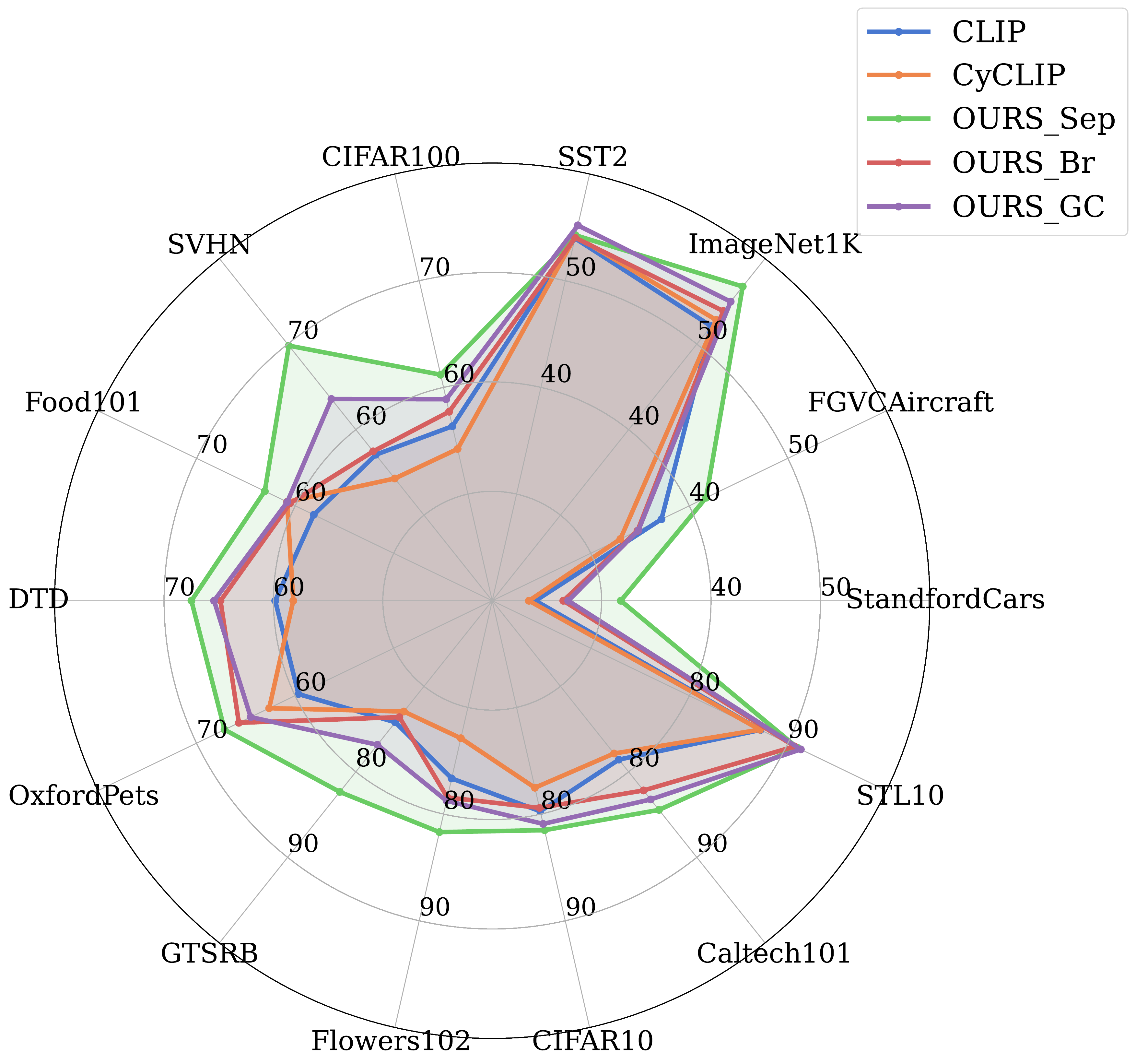}
    \caption{Linear Probing Performance.}
    \label{fig:albef}
  \end{subfigure}
  \vspace{-0.2cm}
  \caption{Visualization of Two-tower-based methods~(\eg CLIP) performance. Each axis represents the performance on a dataset with a certain metric. Each color represents different approach. The larger area that one approach covers, the better overall performance.}
  \label{fig:app_radar}
  \vspace{-0.3cm}
\end{figure*}


\end{document}